%% file: paper.tex
\documentclass[a4paper,11pt]{article}
\usepackage[margin=1in]{geometry}

\usepackage[colorlinks,linkcolor=blue,citecolor=blue]{hyperref}
\usepackage{amsmath,amssymb,amsthm}
\usepackage[round]{natbib}

\usepackage{mathabx}

\usepackage{graphicx} 
\usepackage{subfigure} 

\usepackage{algorithm,algorithmic}
\usepackage{tabularx}
\usepackage{enumitem}

\include{header}

\newcommand{\KL}[2]{D_{\mathrm{KL}}\lr{#1 \,\middle\|\, #2}}
\newcommand{\Regret}{\mathrm{R}_T}
\newcommand{\TV}[2]{D_{\mathrm{TV}}\lr{#1 \,,\, #2}}

\title{Online Learning with Feedback Graphs Without the Graphs}
\author{
  Alon Cohen \qquad Tamir Hazan \qquad Tomer Koren\\
  {\normalsize Technion---Israel Institute of Technology}\\
  \texttt{\normalsize \{alon.cohen,tamir.hazan,tomerk\}@technion.ac.il}
}

\begin{document} 

\maketitle

\begin{abstract}
We study an online learning framework introduced by \citet{MS11} in which the feedback is specified by a graph, in a setting where the graph may vary from round to round and is \emph{never fully revealed} to the learner.
We show a large gap between the adversarial and the stochastic cases. 
In the adversarial case, we prove that even for dense feedback graphs, the learner cannot improve upon a trivial regret bound obtained by ignoring any additional feedback besides her own loss.
In contrast, in the stochastic case we give an algorithm that achieves $\tTheta(\sqrt{\alpha T})$ regret over $T$ rounds, provided that the independence numbers of the hidden feedback graphs are at most $\alpha$.
We also extend our results to a more general feedback model, in which the learner does not necessarily observe her own loss, and show that, even in simple cases, concealing the feedback graphs might render a learnable problem unlearnable.
\end{abstract}

\section{Introduction}

Online learning is a general framework for sequential decision-making under uncertainty. In its most basic form, it can be described as follows. 
A learner has to iteratively choose an action from a set of $K$ available actions, and suffer a loss associated with that action. 
The losses of the actions on each round are assigned in advance by an arbitrary, possibly adversarial, environment.
The learner's goal is to minimize her regret over~$T$ rounds of the game, which is the difference between her cumulative loss and that of the best fixed action in hindsight.

After making each decision, the learner receives some form of feedback about the losses. Traditionally, the literature considers two types of feedback: full feedback \citep{littlestone1994weighted, vovk1990aggregating,cesa1997use}, where the learner observes the losses associated with all of her possible actions, and bandit feedback \cite{AuerCeFrSc02}, where the learner only observes the loss of the action she has actually taken.

Full feedback and bandit feedback are special cases of a general framework introduced by \citet{MS11},
in which the feedback model is specified by a sequence~$G_1,\ldots,G_T$ of \emph{feedback graphs}, one for each round $t$ of the game.
Each feedback graph $G_t$ is a directed graph whose nodes correspond to the learner's $K$ possible actions; a directed edge $u \rightarrow v$ in this graph
indicates that whenever the learner chooses action~$u$ on round $t$, in addition to observing the loss of action~$u$, she also gets to observe the loss associated with the action~$v$ on that round. 

Online learning with feedback graphs was further studied by several authors.
\citet{Alonetal2013}, and subsequently \citet{kocak2014efficient,AlonCDK15}, gave regret-minimization algorithms that achieve $\tO(\sqrt{\alpha T})$ regret, where $\alpha$ is a bound on the \emph{independence numbers} of the graphs $G_1,\ldots,G_T$.
Up to logarithmic factors, their results recover and interpolate between the classic bounds of $\O(\sqrt{T \log K})$ with full feedback \cite{freund1997decision} and $\O(\sqrt{KT})$ with bandit feedback \cite{AuerCeFrSc02,AudibertB09}.
The $\tO(\sqrt{\alpha T})$ bound turns out to be tight for \emph{any} feedback graph (when it is fixed throughout the game and known in advance), in light of a matching lower bound due to \citet{MS11}.

However, all of the optimal algorithms mentioned above require the full structure of the feedback graph in order to operate. 
While some require the entire graph $G_t$ for performing their updates only at the end of round $t$ (e.g., \citealp{Alonetal2013, kocak2014efficient, AlonCDK15}),%
\footnote{More precisely, these algorithms do not need the entire graph but rather the incoming neighborhood of each of the actions for which the associated loss has been observed.}
others actually need the description of $G_t$ at the beginning of the round before making their decision (e.g., \citealp{AlonCGMMS14}).
In fact, none of the algorithms previously proposed in the literature is able to provide non-trivial regret guarantees without the feedback graphs being disclosed.

The assumption that the entire observation system is revealed to the learner on each round, even if only after making her prediction, is rather unnatural.
In principle, the learner need not be even aware of the fact that there is a graph underlying the feedback model; the feedback graph is merely a technical notion  for us to specify a set of observations for each of the possible actions.
Ideally, the only signal we would like the learner to receive following each round is the set of observations that corresponds to the action she has taken on that round (in addition to her own loss).

As a motivating example for situations where receiving the entire observation system is unrealistic, consider the following online pricing problem that faces any vendor selling goods over the internet. 
On each round, the seller has to announce a price for his product. 
Then, a buyer arrives and decides whether or not to purchase the product at this price based on his private value; the only feedback the seller receives is whether or not the buyer purchased the product at the announced price.
However, when a purchase takes place, the seller also knows that the buyer would have bought the product at any price lower than the price that she announced.
While this feedback structure can be thought of as a directed graph over the seller's actions (i.e., prices), the graph itself is never fully revealed to the seller as its structure discloses the buyer's private value.

\subsection{Our contributions}

In this paper, we study online learning with feedback graphs in a setting where the feedback graphs are \emph{never revealed} to the learner in their entirety. That is, in this setting the only feedback available to the learner at the end of round $t$ is the out-neighborhood of her chosen action in the graph~$G_t$, along with the loss associated with each of the actions in this neighborhood and the loss of the action that she chose.
We address the following questions: how this lack of full disclosure affects the learner's regret?
Is it possible to achieve any non-trivial regret guarantee in this setting, i.e., one that improves on the trivial $\O(\sqrt{KT})$ bound?
In particular, can we obtain bounds that scale with the independence numbers of the feedback graphs?

Our main results show that not knowing the entire feedback graphs can have a significant impact on the learner's achievable regret. 
First, we show that in a standard adversarial online learning setting, where we assume nothing about the process generating the losses and the feedback graphs (i.e., both are possibly chosen by an adversary), any strategy of the learner must suffer $\Omega(\sqrt{KT})$ regret in the worst case, even if the independence numbers of $G_1,\ldots,G_T$ are all bounded by some small absolute constant.
Namely, by hiding the feedback graphs from the learner,
the problem surprisingly becomes as hard as the $K$-armed bandit problem, even when the feedback available to the learner is ``almost full'': each of the feedback graphs is ``almost a clique.''
In other words, the side observations received by the learner are effectively useless; she may as well ignore them and use a standard bandit algorithm such as \textsc{Exp3}~\cite{AuerCeFrSc02} to perform optimally.

Second, and in contrast to the adversarial setting, we show that in a stochastic setting where the losses of each action are known to be drawn i.i.d.~from some unknown probability distribution, side observations can still be very useful. 
We show that the learner is able to achieve an optimal regret bound of the form $\tO(\sqrt{\alpha T})$, even if the graphs $G_1,\ldots,G_T$ are chosen adversarially and are never fully revealed to the learner, as long as their independence numbers are all bounded by $\alpha$.
We give an efficient elimination-based algorithm achieving this bound, that does not require knowing the value of $\alpha$ in advance.
This result is optimal up to logarithmic factors, even when the feedback graph is fixed throughout the game and known in advance, due to a lower bound of \citet{MS11}.

For our algorithm in the stochastic case, we also prove a distribution-dependent regret bound that scales logarithmically with $T$.
The bound we prove is of the form $\O(\sum_{v \in V'} (1/\Delta_v) \log T)$, where $\Delta_v$ is the gap of action $v$, and $V'$ is the subset of $\tO(\alpha)$ actions with smallest gaps.
%
This bound is a substantial improvement over standard regret bounds of stochastic multi-armed bandit algorithms such as UCB~\cite{auer2002finite}:
whereas the regret of the latter algorithms is typically bounded by a sum $\sum_{v \in V} (1/\Delta_v)$ taken over \emph{all} $K$ actions, the sum in our bound is taken only over the subset of $\tO(\alpha)$ actions with the smallest gaps.
Again, this result cannot be improved even when the feedback graph is fixed throughout the game, and has an optimal dependence on $\alpha$ as well as on the gaps $\Delta_v$, thus resolving an open question of \citet{AlonCGMMS14}.

Finally, we extend our results to a more general feedback model recently studied by \citet{AlonCDK15}, in which the learner does not necessarily observe her own loss after making predictions (namely, each action may or may not have a self-loop in each feedback graph).
\citet{AlonCDK15} gave a necessary and sufficient condition for attaining $\Theta(\sqrt{T})$ regret in this more general model---a graph-theoretic condition they call strong observability.
The extension of our results to their model bears some surprising consequences: for example, even in the strongly observable case with only two actions, not revealing the entire feedback graphs to the learner might make the problem unlearnable!
Nevertheless, in the case of stochastic losses, our positive results do extend to the more general feedback model.

\subsection{Additional related work}

Online learning with feedback graphs was previously considered in the stochastic setting by \citet{caron2012leveraging}, who gave results depending on the graph clique structure.
Their analysis, however, only applies when the feedback graph is fixed throughout the game, and can only bound the regret in terms of a quantity akin to the clique-partition number of this graph, which is always larger than its independence number (the gap between the two can be very large; see \citealp{AlonCGMMS14}).

More recently, \citet{wu2015online} and \citet{kocak2016online} have investigated a noisy version of the feedback graph model, where feedback is specified by a weighted directed graph with edge weights indicating the quality (e.g., the noise level or variance) of the feedback received on adjacent vertices.
\citet{wu2015online} provided finite-time problem-dependent lower bounds for this setting; \citet{kocak2016online} generalized the notion of independence number to the noisy case and gave new efficient algorithms in this setting.

\section{Setup and Main Results}


We consider a general online learning model with graph-structured feedback, which can be described as a game between a learner and an environment that proceeds for $T$ rounds. 
Before the game begins, the environment privately determines a sequence of loss 
functions $\ell_1,...,\ell_T : V \mapsto [0,1]$ defined over a set $V = \set{1,...,K}$ of $K$ actions, which we view as a sequence of loss vectors $\ell_1,...,\ell_T \in [0,1]^K$. 
In addition, the environment fixes a sequence of directed graphs $G_1,\ldots,G_T$ over $V$ as vertices.

We will consider two different cases, that we refer to as the adversarial setting and the stochastic setting:
\begin{itemize}
\item
In the adversarial setting, the loss vectors $\ell_1,...,\ell_T$ and the feedback graphs $G_1,\ldots,G_T$ are chosen by the environment in an arbitrary way. 
\item
In the stochastic setting, the environment privately selects a loss distribution $\D$ over $[0,1]^K$ and an arbitrary sequence $G_1,\ldots,G_T$;
thereafter, the loss vectors $\ell_1,\ldots,\ell_T$ are sampled i.i.d.~from~$\D$. An important property of this setting is that the loss vectors are \emph{statistically independent} from the feedback graphs.
\end{itemize}

Iteratively on rounds~$t = 1,2,...,T$, the learner randomly chooses an action $v_t \in V$ and incurs the loss $\ell_t(v_t)$. 
At the end of each round $t$, the learner receives a feedback comprised of $\set{ ( v, \ell_t(v) ) : (v_t \rightarrow v) \in G_t) }$, that includes the loss $\ell_t(v_t)$ incurred by the learner (i.e., we assume that $(v \rightarrow v) \in G_t$ for all $t$ and $v \in V$).
In words, the learner observes the losses associated with $v_t$ and the actions in the out-neighborhood of $v_t$ in the feedback graph~$G_t$. 
The feedback graph $G_t$ itself \emph{is never revealed} in its entirety to the learner.

The goal of the learner throughout the $T$ rounds of the game is to minimize her expected regret, which is defined as 
\begin{align}
\label{eq:regret}
\Regret 
\eq 
\EE{ \sum_{t=1}^T \ell_t(v_t) \,-\, \sum_{t=1}^T \ell_t(v^\star) } ~,
\end{align}
where $v^\st = \min_{v \in V} \E[\sum_{t=1}^T \ell_t(v)]$ is the best action~in hindsight. Here, the expectations are taken over the random choices of the learner and, in the stochastic setting, also over the randomness of the losses.

For the stochastic setting we require additional notation.
For each $v \in V$, we denote by $\mu(v)$ the mean of the loss of action $v$ under $\D$. 
We denote $\mu^\st = \mu(v^\st)$, and let $\Delta_v = \mu(v) - \mu^\star$ for all $v \in V$. 
We refer to~$\Delta_v$ as the \emph{gap} of action $v$, and assume for simplicity that $v^\star$ is unique so that $\Delta_v > 0$ for all $v \ne v^\star$.


For stating our results, we need a standard graph-theoretic definition.
An \emph{independent set} in a graph $G=(V,E)$ (either directed or undirected) is a set of vertices that are not connected by any edges. Namely, $S \subseteq V$ is independent if for any $u,v \in S$, $u \neq v$, it holds that $(u \rightarrow v) \notin E$. The \emph{independence number} $\alpha(G)$ of $G$ is the size of the largest independent set in~$G$.

\subsection{Main results}

We now state the main results of this paper.
Our first result deals with the adversarial case and shows that when the feedback graphs are not revealed to the learner at the end of each round, her regret might be very large even when the independence numbers of the graphs are small---they are all bounded by a constant.

\begin{theorem}
\label{thm:main-lower}
In the adversarial setting, any online learning algorithm must suffer at least~$\Omega(\sqrt{KT})$ regret in the worst case, even when all feedback graphs $G_1,\ldots,G_T$ have independence numbers $\le O(1)$.
\end{theorem}

The lower bound in the theorem is tight: it can be matched by simply running a standard bandit algorithm (e.g., \textsc{Exp3}, \citealp{AuerCeFrSc02}), ignoring all observed feedback besides the loss of the action played.

Our next result shows that in the stochastic case, the learner is still able to attain non-trivial regret despite the fact that the feedback graphs are never fully revealed to her.
\begin{theorem}
\label{thm:main-upper}
In the stochastic setting, \cref{alg:stochastic} described in \cref{sec:stoch} attains an expected regret of at most $\tO(\sqrt{\alpha T})$, provided that the independence numbers of the graphs $G_1,\ldots,G_T$ are all bounded by $\alpha$.
\end{theorem}
This regret bound is optimal up to logarithmic factors, since the lower bound of $\Omega(\sqrt{\alpha T})$ found in \citet{MS11} applies in our stochastic setting.

In the stochastic setting we also give a distribution-dependent analysis of \cref{alg:stochastic} which depends on the gaps of the actions under the distribution $\D$.

\begin{theorem}
\label{thm:main-upper2}
In the stochastic setting, \cref{alg:stochastic} described in \cref{sec:stoch} attains an expected regret of
\[
\O\lr{\sum_{v \in V'} \frac{1}{\Delta_v} \log T} ~,
\]
where $V'$ is the set of $\tO(\alpha)$ actions with the smallest gaps (excluding $v^\star$), provided that the the independence numbers of the graphs $G_1,\ldots,G_T$ are all bounded by~$\alpha$.
\end{theorem}

We also extend our results to a more general class of feedback graphs, in which each vertex may or may not have a self-loop.
For the statements of these additional results, see \cref{sec:extensions}.

\subsection{Discussion of the results}


Our results show that there is a large gap between the achievable regret rates in the adversarial and stochastic settings, in terms of the dependence on the properties of the feedback graphs.

In the adversarial case, the environment is free to simultaneously choose the sequences of loss values and feedback graphs in conjunction with each other; for example, they can be drawn from a \emph{joint} distribution over sequences of loss values \emph{and} sequences of directed graphs.
The environment may use this freedom to manipulate the feedback observed by the learner and bias her observations in a malicious way.
In the stochastic setting, on the other hand,
the loss values are drawn from the underlying distribution only after the environment commits to some arbitrary sequence of graphs, so that the feedback graphs are \emph{probabilistically independent} of the realizations of the losses.

In fact, as our arguments in \cref{sec:adver} reveal, there exists a randomized construction of loss vectors and feedback graphs that inflicts $\Omega(\sqrt{KT})$ on any learner, in which the loss vectors are i.i.d. 
However, the stochastic process that generates the feedback graphs in that construction is correlated with the actual realizations of the i.i.d.~losses.
This is a crucial aspect of our construction, as implied by our upper bound in the stochastic case.


\section{Lower Bound for Adversarial Losses}
\label{sec:adver}

In this section we deal with the adversarial setting and prove \cref{thm:main-lower}: we show an $\Omega(\sqrt{KT})$ lower bound on the performance of any online learning algorithm, where both the losses of the actions and the feedback graphs can be chosen arbitrarily.

Let us sketch the idea behind the lower bound, and defer the formal details to \cref{sec:additionalproofs}.
By Yao's minimax principle, 
in order to prove a lower bound on the learner's regret it is enough to demonstrate a randomized strategy for the environment that forces any deterministic learner to incur $\Omega(\sqrt{KT})$ regret.
We construct our environment's strategy as follows. 

First, before the game begins, the environment chooses an action $v^\star$ uniformly at random from $V$. At each round, the loss of all actions $v \ne v^\star$ is distributed Bernoulli($1/2$), while the loss of action $v^\star$ is distributed Bernoulli($1/2 - \epsilon$) with $\epsilon = (1/8) \sqrt{K / T}$. 
All of the loss values in the construction are drawn independently of each other.

The feedback graphs $G_1,\ldots,G_T$ are chosen i.i.d.~from the following distribution. Any edge $u \rightarrow v$ for $v \neq v^\star$ appears with probability $1-2 \epsilon$ independently from all other edges and the losses of the actions. Edges of the form $u \rightarrow v^\star$ appear mutually independently given the loss of action $v^\star$: if the loss of $v^\star$ is 1, each edge appears with probability 1;
if the loss of $v^\star$ is 0, each edge appears with probability $(1-2 \epsilon)/(1 + 2 \epsilon)$. 
See \cref{fig:probabilities} for a summary of the edge probabilities in this construction.

\renewcommand{\arraystretch}{1.35}
\newcolumntype{C}{>{\centering\arraybackslash}X}%

\begin{figure}[t]
\begin{center}
\begin{tabularx}{0.45\linewidth}{ p{1.8cm}|C|C|C }
   & $u \not\rightarrow v^\star$ & $u \rightarrow v^\star$ & \\\hline 
$\ell_t(v^\star) = 0$ & \cellcolor{lightgray}$2\epsilon$ & $\half-\epsilon$ & \cellcolor{lightgray}$\half + \epsilon$ \\ \hline
$\ell_t(v^\star) = 1$ & \cellcolor{lightgray}0 & $\half-\epsilon$ & \cellcolor{lightgray}$\half - \epsilon$ \\ \hline
   & $2\epsilon$ & $1-2\epsilon$ & 
\end{tabularx}
\hspace{0.05\linewidth}
\begin{tabularx}{0.45\linewidth}{ p{1.8cm}|C|C|C }
   & $u \not\rightarrow v$ & $u \rightarrow v$ & \\\hline 
$\ell_t(v) = 0$ & \cellcolor{lightgray}$\epsilon$ & $\half-\epsilon$ & \cellcolor{lightgray}$\half$ \\ \hline
$\ell_t(v) = 1$ & \cellcolor{lightgray}$\epsilon$ & $\half-\epsilon$ & \cellcolor{lightgray}$\half$ \\ \hline
   & $2\epsilon$ & $1-2\epsilon$ & 
\end{tabularx}
\end{center}
\vspace{-0ex}
\caption{Summary of the joint distribution of the loss of action $v^\star$ and an edge between $u$ and $v^\star$ (left), and of the joint distribution of the loss of action $v \neq v^\star$ and an edge between $u$ and $v$ (right). The grayed-out entries indicate probabilities that cannot be estimated by the learner; the remaining entries do not permit the learner to distinguish between $v^\star$ and $v$.}
\label{fig:probabilities}
\end{figure}

The idea behind the construction is as following. Suppose that the learner plays some action $u \neq v^\star$, the distributions of the observed losses of every other actions are identical, including that of $v^\star$.
In other words, her only option of finding $v^\star$ is by sampling it directly and observing its loss. Hence, the construction is capable of simulating a $K$-armed bandit problem whose minimax regret is $\Omega(\sqrt{KT})$.

For the construction above, we prove the following theorem. The proof itself is deferred to \cref{sec:lowerboundproof}.

\begin{theorem}
\label{thm:lowerbound}
Assume that $K \ge 2$ and $T \ge K^2$. Any deterministic learner must suffer an expected regret of at least $(1/32) \sqrt{K T}$ against the environment constructed above. 
\end{theorem}

To show that \cref{thm:main-lower} holds, we need to show that the learner suffers a large regret against an environment that selects feedback graphs with constant independence numbers.
While the independence numbers of the graphs that we have constructed might, in principle, be large, we can show that with very high probability they are uniformly bounded by a constant.

\begin{lemma}
\label{lemma:rgis}
Suppose that $|V| = K \ge 2$ and $T \ge K^2$. Let $G_1,...,G_T$ be a sequence of graphs as constructed above.
With probability at least $1 - \epsilon / 8$, the independence numbers of all graphs are at most 9.
\end{lemma}

\cref{thm:main-lower} now follows by combining \cref{thm:lowerbound} and \cref{lemma:rgis}; for the technical details, see \cref{sub:proofofmainlower}.

\section{Algorithms for Stochastic Losses}
\label{sec:stoch}


In this section we present and analyze our algorithm for the stochastic setting. 
The algorithm, given in \cref{alg:stochastic}, is reminiscent of elimination-based algorithms for the stochastic multi-armed bandit problem (e.g., \citealp{even2002pac,karnin2013almost}).
For this algorithm, we prove the following guarantee on the expected regret, which implies \cref{thm:main-upper}.

\begin{theorem}
\label{thm:alg1regretbound}
Assume that $K \ge 2$. Suppose that \cref{alg:stochastic} is run on a sequence of feedback graphs with independence numbers $\le \alpha$. Then the expected regret of the algorithm is at most $\tO(\sqrt{\alpha T})$.
\end{theorem}

\cref{alg:stochastic} works in phases $r=1,2,\ldots$. It maintains a subset of actions $V_r$, where initially $V_1 = V$. At each phase~$r$, the algorithm estimates the mean losses of all actions in~$V_r$ to within $\epsilon_r$ accuracy, by invoking a procedure called \textsc{AlphaSample} $n_r$ times.
It then filters out from $V_r$ the actions that are known to be $2\epsilon_r$-suboptimal with sufficient confidence, and repeats this process, decreasing the accuracy parameter $\epsilon_r$ after each phase.

The key for achieving optimal regret lies in the the procedure \textsc{AlphaSample}, that appears as \cref{alg:efficientsampling}.
Each call to this procedure allows us to observe the losses of all actions in $V_r$ once, while spending only~$\tO(\alpha)$ rounds in expectation.
The exact details of \textsc{AlphaSample} are discussed in \cref{sub:graphsample} below, and here we just state its guarantee.

\begin{lemma}
\label{lemma:alg2expectedtime}
\textsc{AlphaSample} returns one sample of the loss of each action in $V_r$ and terminates after at most $10 \alpha \log K$ rounds of the game in expectation, provided that the independence numbers of all feedback graphs $G_1,\ldots,G_T$ are at most $\alpha$.
\end{lemma}

To prove \cref{thm:alg1regretbound} we need one additional lemma. It shows that, at each phase, the elimination procedure of the algorithm succeeds with high probability. Namely, after phase $r$, the algorithm is left with actions that are at most $4 \epsilon_r$-suboptimal.

\begin{lemma}
\label{lemma:meanapproximation}
For all $r$, with probability at least $1 - 1/T$ we have $\mu(v) \le \mu^\star + 4 \epsilon_r$ for all $v \in V_{r+1}$.
\end{lemma}
\begin{algorithm}[t]
\caption{}
\label{alg:stochastic}
\begin{algorithmic}
\STATE \textbf{input} Set $V$ of $K$ actions, number of rounds $T$
\STATE \textbf{initialize} $r \gets 1$, $V_1 = V$, $\epsilon_1 = 1/4$
\WHILE{$|V_r| > 1$ and $T$ rounds have not elapsed}
\STATE Set $n_r = \lceil 2 \log(2KT) / \epsilon_r^2 \rceil$
\STATE Invoke \textsc{AlphaSample}($V_r$) for $n_r$ times, and 
\STATE \quad compute empirical mean $m_r(v)$ of each action
\STATE \quad $v \in V_r$ using collected samples
\STATE Compute $m^\star_r = \min_{v \in V_r} m_r(v)$
\STATE Eliminate actions: 
\STATE \quad $V_{r+1} = \{ v \in V_r : m_r(v) \le m^\star_r + 2 \epsilon_r \}$
\STATE Set $\epsilon_{r+1} = \epsilon_r / 2$, $r \gets r + 1$
\ENDWHILE
\STATE Play the action left in $V_r$ until $T$ rounds have passed
\end{algorithmic}
\end{algorithm}
We can now proceed with the proof of the theorem.
\vspace{-1ex}
\begin{proof}[Proof of \cref{thm:alg1regretbound}]
Let us start by bounding the number of phases $R$ the algorithm makes. Let the random variable $T_r$ denote the number of game rounds elapsed during phase $r$. Since the algorithm runs for $T$ rounds we must have that 
\begin{align}
\label{eq:mbound}
\sum_{r=1}^R T_r \le T~.
\end{align}
In particular, since \textsc{AlphaSample} takes at least one round to complete, we have that $T_r \ge n_r \ge 2 \log(2KT) 4^{r+1}$ and we get the crude bound of 
\begin{equation}
\label{eq:mcrudebound}
R \le \bar r = \frac{1}{2} \log_2 \left( \frac{3T}{32 \log(2KT)} + 1 \right)~.
\end{equation}

We turn to bound the expected regret of the algorithm. 
By \cref{lemma:meanapproximation} and the union bound, the total probability of failure of the mean estimations is at most $\bar r/T$. 
Then the expected regret of the algorithm is at most the expected regret conditioned on the success of the estimation of the means plus $(\bar r / T) \cdot T = \bar r = O(\log T)$ by \cref{eq:mcrudebound}, and since the regret is bounded by $T$ with probability 1. Thus it remains to bound the regret conditioned on the success of the mean estimations.

For convenience, define $\epsilon_0 = 1/2$. On phase $r$, by \cref{lemma:meanapproximation} we have an instantaneous regret of at most $4 \epsilon_{r-1} = 8 \epsilon_r$ per round. If only one action is left in $V_r$ then it must be $v^\star$ and therefore after the final phase the algorithm suffers zero instantaneous expected regret. 
Overall, the expected regret is at most,
\begin{align*}
\E \left[ \sum_{r=1}^R T_r \cdot 8 \epsilon_r \right] &\le 8 \sqrt{\E \left[ \sum_{r=1}^R T_r \right]} \cdot \sqrt{\E \left[\sum_{r=1}^R T_r \epsilon_r^2  \right]}
\end{align*}
by the Cauchy-Schwartz inequality. Note that $\sum_{r=1}^R T_r \le T$ by \cref{eq:mbound}.
Additionally, by \cref{lemma:alg2expectedtime} each call to \textsc{AlphaSample} spends at most $m = 10 \alpha \log K $ rounds in expectation and thus $\E[T_r] \le m n_r$. Hence,
\begin{align*}
\E \left[ \sum_{r=1}^R T_r \epsilon_r^2 \right] &\le \sum_{r=1}^{\bar r} m n_r \epsilon_r^2 \le m \bar r (2 \log(2KT) + 1) ~.
\end{align*}
The first inequality holds since the number of phases is at most $\bar r$. 
The right-hand side is $\O(\alpha \log(K) \log^2 (K T))$ by \cref{eq:mcrudebound} and the definition of~$m$.
\end{proof}

\subsection{Gap-based analysis}

We can also provide a distribution-dependent analysis of \cref{alg:stochastic} that yields a logarithmic regret bound, albeit with an explicit dependence on the gaps $\Delta_v$.

Denote by $V^{(n)}$ the set of $n$ actions with smallest gaps, excluding $v^\star$ and breaking ties arbitrarily.\footnote{If $n > K-1$, we simply take $V^{(n)}$ to be the set of all actions besides $v^\star$.}
Our main result in this section is the following theorem, which gives \cref{thm:main-upper2}.
Recall that we assume $v^\star$ is the unique optimal action, and so the gaps of all other actions are positive.

\begin{theorem}
\label{thm:alg1gapbasedregretbound}
Suppose that $K \ge 2$ and $T \ge K$, and let $\tau = \lceil 10 \alpha \log K \rceil$. Suppose that \cref{alg:stochastic} is run on a sequence of feedback graphs with independence numbers  $\le \alpha$.
Then the expected regret of \cref{alg:stochastic} is at most
\[
\O\lr{ \sum_{v \in V^{(\tau)}} \frac{1}{\Delta_v} \log T } ~.
\]
\end{theorem}

We can explain the intuition behind the bound as follows. 
Each call to \textsc{AlphaSample} spends at most $\tO(\alpha)$ rounds while producing samples of all $K$ actions. Thus, in the worst case, after a quick pruning phase the algorithm is left with the ``hardest'' $\tau = \tO(\alpha)$ actions and has to tell them apart; in this last phase, the additional observations provided by the feedback graphs might not help the algorithm at all (e.g., the remaining $\tau$ actions might form an independent set in all graphs).
Let us turn to the proof of the theorem.

\begin{proof}[Proof of \cref{thm:alg1gapbasedregretbound}]
As in the proof of \cref{thm:alg1regretbound}, we have that the expected regret of the algorithm is at most the expected regret conditioned on the success of the mean estimations plus $O(\log T)$, and thus it remains to bound the regret conditioned on the success of the mean estimations.

Conditioned on the success of the algorithm, the regret of the algorithm is at most the regret of an algorithm that has finished running with $V_r = \{v^\star\}$. Thus we can assume that $T$ is large enough for that to happen.

If $\tau \le K-1$, we begin by bounding the regret until the algorithm eliminates all actions besides the ones in $V^{(\tau)}$. Let $\bar \Delta$ be the largest gap of an action from $V^{(\tau)}$. Let $\bar r = \lfloor \log_2(2 / \bar \Delta) \rfloor$. Thus, it takes $\bar r+1$ phases in order for $\epsilon_r$ to be less than $\bar \Delta / 4$. The regret up to round $\bar r$ is bounded using the following lemma.

\begin{lemma}
\label{lemma:gapfirstphaseregret}
Let $m = 10 \alpha \log K$. The expected regret of \cref{alg:stochastic} up to round $\bar r$ is at most
\[
\frac{128m}{\bar \Delta} \log (2KT)~.
\]
\end{lemma}

We proceed with the analysis of the expected regret after phase $\bar r$. This is given by this next lemma.

\begin{lemma}
\label{lemma:gapsecondphaseregret}
The expected regret of \cref{alg:stochastic} from round $\bar r+1$ until the end of the game is at most
\[
\sum_{v \in V^{(\tau)}} \frac{128}{\Delta_v} \log (2KT)~.
\]
\end{lemma}

If $\tau > K-1$ then the regret of the algorithm is given by \cref{lemma:gapsecondphaseregret}. Otherwise, the proof of the theorem is completed by noticing that the regret of the algorithm up to round $\bar r$ is at most the regret from round $\bar r+1$ thereafter.
Since $\bar \Delta \ge \Delta_v$ for all $v \in V^{(\tau)}$ we get that
\[
\frac{m}{\bar \Delta} \le \frac{m}{|V^{(\tau)}|} \sum_{v \in V^{(\tau)}} \frac{1}{\Delta_v} \le \sum_{v \in V^{(\tau)}} \frac{1}{\Delta_v}~,
\]
by definition of $m$ and $V^{(\tau)}$.
This in total gives a regret bound of $O(\sum_{v \in V^{(\tau)}} (1 / \Delta_v) \log(KT))$.
Finally, we use our assumption that $T \ge K$ to simplify the bound.
\end{proof}

\begin{proof}[Proof of \cref{lemma:gapfirstphaseregret}]
By \cref{lemma:alg2expectedtime}, each call to \textsc{AlphaSample} spends at most $m$ rounds in expectation. 
By \cref{lemma:meanapproximation}, the instantaneous regret for each round on phase $r$ is at most $4 \epsilon_{r-1} = 8 \epsilon_r$.
Then the expected regret up to round $\bar r$ is at most 
\[
\sum_{r=1}^{\bar r} m \cdot n_r \cdot 8 \epsilon_r \le 32 m \log(2KT) \sum_{r=1}^{\bar r} \frac{1}{\epsilon_r} ~,
\]
and we have
\[
\sum_{r=1}^{\bar r} \frac{1}{\epsilon_r} = \sum_{r=1}^{\bar r} 2^{r+1} 
\le 2^{\bar r+2} \le \frac{4}{\bar \Delta}~. \qedhere
\]
\end{proof}

\begin{proof}[Proof of \cref{lemma:gapsecondphaseregret}]
Let us denote $\bar r_v = \lfloor \log_2 (2 / \Delta_v) \rfloor$, the number of phases until $v$ is removed from $V_r$. Let $w$ be the action with the minimum nonzero gap.
We shall assume that the game is finished after $\bar r_{w}$ phases. 

Note that after we have eliminated all actions not in $V^{(\tau)}$, each call to \textsc{AlphaSample} is finished after at most $|V_r|$ steps. Thus, the expected regret for the remaining phases is at most
\begin{align*}
&\sum_{r = \bar r+1}^{\bar r_{w}} \frac{32 \log(2KT)}{\epsilon_r} |V_r| = 32 \log(2KT) \sum_{v \in V^{(\tau)}} \sum_{r = \bar r+1}^{\bar r_v} \frac{1}{\epsilon_r}~,
\end{align*}
and, for all $v \in V^{(\tau)}$,
\[
\sum_{r = \bar r+1}^{\bar r_v} \frac{1}{\epsilon_r} \le \sum_{r = 0}^{\bar r_v} 2^{r+1} = 2^{\bar r_v + 2} \le \frac{4}{\Delta_v}~. \qedhere
\]
\end{proof}

\subsection{Efficient sampling scheme}
\label{sub:graphsample}

In this section, we discuss the \textsc{AlphaSample} randomized sampling procedure.
This procedure allows us to collect one sample of the loss for each action while spending only $\tO(\alpha)$ rounds in expectation.
\textsc{AlphaSample} is described in \cref{alg:efficientsampling}.

Let us now explain the intuition behind the procedure.
%
At each round, the procedure samples the loss of an action uniformly at random from a subset of actions $U$. 
As each sample is uniform over $U$, the procedure observes the losses of $\Omega(|U|/\alpha)$ actions in expectation. 
The actions that have been observed are then removed from $U$ and the process continues recursively until $U$ is empty. 
This phase is complete after an expected $\tO(\alpha)$ rounds.
\begin{algorithm}[t]
\caption{\textsc{AlphaSample}}
\label{alg:efficientsampling}
\begin{algorithmic}
\STATE \textbf{input} Set of actions $U \subseteq V$
\STATE \textbf{initialize} $S \gets \emptyset$
\WHILE{$|U| > 0$}
\STATE Play an action $u \in U$ uniformly at random, 
\STATE \quad and let $W(u)$ be the set of actions observed
\STATE Collect samples of losses of each $w \in W(u)$ into $S$
\STATE Update $U \gets U \setminus W(u)$
\ENDWHILE
\STATE \textbf{return} $S$
\end{algorithmic}
\end{algorithm}

The main result regarding \textsc{AlphaSample} is the following theorem, from which \cref{lemma:alg2expectedtime} would follow immediately (see \cref{sub:proofalg2expectedtime}).

\begin{theorem}
\label{thm:alg2analysis}
\cref{alg:efficientsampling} returns one sample of the loss of each action in $U$ and terminates after at most $4 \alpha \log(K/\delta)$ rounds with probability at least $1 - \delta$, provided that all feedback graphs $G_1,\ldots,G_T$ have independence numbers $\le \alpha$.
\end{theorem}

To analyze the number of rounds that the algorithm spends, we shall define the following random process. Consider an infinite sequence $U_1,U_2,...$ such that $U_1 = U$. For every $r > 0$, if $U_r$ is not empty we sample an action uniformly at random from $U_r$, and we let $U_{r+1}$ be $U_r$ after removing the actions whose losses were observed. Otherwise, we let $U_{r+1}$ be the empty set.

The following lemma lower bounds the expected number of actions whose losses are observed at each iteration of the process. 

\begin{lemma}
\label{lemma:expectedverticesseen}
Let $r > 0$. Let $N$ be the number of actions seen when sampling uniformly at random from $U_r$. Then, $\E[N | U_r] \ge |U_r|/(2 \alpha)$. 
\end{lemma}

The main tool used in the proof of the lemma is the following version of Tur\'{a}n's theorem~(see, e.g., \citealp{alon2011probabilistic}). 

\begin{theorem}[Tur\'{a}n]	
\label{thm:turan}
Let $G = (V,E)$ be an undirected graph and $\alpha$ be the independence number of $G$. Then,
\[
\alpha \ge \frac{|V|}{1 + 2 |E| / |V|}~.
\]
\end{theorem}
%
%
\vspace{-0ex}
\begin{proof}[Proof of \cref{lemma:expectedverticesseen}]
Fix some feedback graph $G = (V, E)$ with independence number $\le \alpha$, and let $d_{\text{out}}(v)$ be the out-degree of vertex $v$. Note that the independence number of the subgraph over $U$ can only decrease, namely it is also at most $\alpha$. As such, we shall think of $d_{\text{out}}(v)$ as the out-degree of $v$ in the subgraph. 

We would like to apply Tur\'{a}n's theorem to the subgraph, which is a directed graph. We do so by constructing an undirected version of the subgraph, namely one in which we ignore the orientation of the edges. Note that the number of edges in the undirected version can only decrease.
Therefore,
\[
\E[N | U_r] = 1 + \frac{1}{|U_r|} \sum_{v \in U_r} d_{\mathrm{out}}(v) = 1 + \frac{|E|}{|U_r|} \ge \frac{|U_r|}{2 \alpha}~,
\]
where the inequality follows from Tur\'{a}n's theorem (\cref{thm:turan}).
\end{proof}
%
\begin{proof}[Proof of \cref{thm:alg2analysis}]
By the construction of the random process, the probability that \cref{alg:efficientsampling} spends more than $t$ rounds of the game is exactly the probability that $U_{t+1}$ is not empty.
To bound this probability we claim that for any $r > 0$,
\begin{equation}
\label{eq:versionspaceshrinkage}
\E[|U_{r+1}|] \le K \exp(-r/(2\alpha))~.
\end{equation}
Indeed, fix some $i > 0$. By \cref{lemma:expectedverticesseen} we have that $\E[|U_{i+1}| | U_i] = |U_i| - \E[N | U_i] \le |U_i| (1 - 1 / (2 \alpha))$.
Taking expectation with respect to $U_i$ and then applying this argument recursively, we get that $\E[|U_{r+1}|] \le |U_1| (1 - 1 / (2 \alpha))^r \le K \exp(-r / (2 \alpha))$.

Now, let $t_1 = \lfloor 2 \alpha \log(K/\delta) \rfloor + 1$. We will show that the probability that $U_{t_1 + 1}$ is not empty is at most $\delta$. By Markov's inequality and \cref{eq:versionspaceshrinkage},
\begin{align*}
\Pr[|U_{t_1+1}| > 0] &\le \E[|U_{t_1+1}|]
\le K \exp(-t_1/(2\alpha)) < \delta ~.
\end{align*}
%
To conclude, with probability at least $1 - \delta$, the number of rounds that the algorithm spends is at most $t_1 \le 4 \alpha \log (K / \delta)$, 
since $K \ge 2$ by assumption. 
\end{proof}

\section{Beyond Bandit Feedback}
\label{sec:extensions}

In this section we extend our results to a more general class of feedback graphs. 
In particular, we no longer assume that the learner automatically gets to observe the loss of the action that she chose. Instead, we allow the graphs to have self-loops, namely edges of the form~$v \rightarrow v$. 
The absence of self-loops at individual actions allows for feedback models that are not necessarily more informative than the bandit model.

Recently, \citet{AlonCDK15} have studied this more general feedback model and divided feedback graphs into three categories: unobservable graphs, for which the induced problem is not learnable; weakly observable graphs, for which $\tTheta(T^{2/3})$ regret is achievable;
and strongly observable graphs, for which it is possible to attain $\tTheta(\sqrt{T})$ regret. 
Their results assume that the feedback graphs are available to the learner, at least after making each prediction.
Here, we revisit their results assuming that the graphs are never fully revealed to the learner.

We begin by recalling the definitions of observability of \citet{AlonCDK15}.
A vertex in a directed graph is \emph{observable} if it has at least one incoming edge. A vertex is \emph{strongly observable} if it has either a self-loop or edges incoming from \emph{all} other vertices.
A vertex is \emph{weakly observable} if it is observable but not strongly observable.
A graph is \emph{observable} if all of its vertices are observable, and it is \emph{strongly observable} if all of its vertices are strongly observable. 
A graph is \emph{weakly observable} if it is observable but not strongly observable.
Note that a graph with self-loops at all vertices is necessarily strongly observable. 

Let us now discuss our results; below, we only give the main ideas and sketch the proofs, deferring details to the full version of the paper.

\subsection{Strongly observable graphs}

In the adversarial setting, we show that the problem might be  \emph{unlearnbable} even with strongly observable graphs; formally, we prove (see \cref{sub:proofofstronglylowerbound}):
\begin{theorem}
\label{thm:stronglylowerbound}
In the adversarial setting, any algorithm must suffer at least~$T / 16$ regret in the worst case, even when $G_1,\ldots,G_T$ are all strongly observable.
\end{theorem}
\vspace{-1ex}
\begin{proof}[Proof sketch]
Consider a problem over two actions, $u$ and $v$. The environment chooses one of two distributions over the choice of the loss of action $v$, the edge $u \rightarrow v$ and the self-loop $v \rightarrow v$, that are summarized in \cref{fig:stronglyprobabilities}. 
Each cell in the table is split into two, where the left half is for the first distribution, and the right half is for the second distribution. The two rightmost columns indicate the marginal distributions between the loss of action $v$ and either the edge $u \rightarrow v$ or the self-loop at $v$.
Additionally, the action $u$ always has a self-loop and its loss is constantly $1/2$.

\begin{figure}[t]
\begin{center}
\begin{tabular}{ c |c | c | c | c | c | c || c | c | c | c}
   & \multicolumn{2}{c|}{$u \not\rightarrow v$} & \multicolumn{2}{c|}{$u \rightarrow v$} & \multicolumn{2}{c||}{$u \rightarrow v$} & \multicolumn{2}{c|}{} &  \multicolumn{2}{c}{} \\
$\ell_t(v)$ & \multicolumn{2}{c|}{$v \rightarrow v$} & \multicolumn{2}{c|}{$v \rightarrow v$}  & \multicolumn{2}{c||}{$v \not\rightarrow v$} & \multicolumn{2}{c|}{$u \rightarrow v$} & \multicolumn{2}{c}{$v \rightarrow v$} \\\hline 
$0$ & \cellcolor{lightgray}$\frac{1}{4}$ & \cellcolor{lightgray}$0$ & \cellcolor{lightgray}$\frac{1}{8}$ & \cellcolor{lightgray}$\frac{3}{8}$ & \cellcolor{lightgray}$\frac{1}{4}$ &\cellcolor{lightgray} $0$ & $\frac{3}{8}$ & $\frac{3}{8}$ & $\frac{3}{8}$ & $\frac{3}{8}$ \\ \hline
$1$ & \cellcolor{lightgray}$0$ & \cellcolor{lightgray}$\frac{1}{4}$ & \cellcolor{lightgray}$\frac{3}{8}$ & \cellcolor{lightgray}$\frac{1}{8}$ & \cellcolor{lightgray}$0$ & \cellcolor{lightgray}$\frac{1}{4}$ & $\frac{3}{8}$ & $\frac{3}{8}$ & $\frac{3}{8}$ & $\frac{3}{8}$ \\ \hline\hline
   & \multicolumn{2}{c|}{\cellcolor{lightgray}$\frac{1}{4}$} & \multicolumn{2}{c|}{\cellcolor{lightgray}$\half$} & \multicolumn{2}{c||}{\cellcolor{lightgray}$\frac{1}{4}$} & \multicolumn{2}{c|}{$\frac{3}{4}$} & \multicolumn{2}{c}{$\frac{3}{4}$} \\\hline
\end{tabular}
\end{center}
\vspace{-0ex}
\caption{Summary of the joint distributions of the loss of action $v$, a self-loop at $v$, and an edge between $u$ and $v$. The grayed-out entries indicate probabilities that cannot be estimated by the learner; the remaining entries do not permit the learner to distinguish between the distributions. }
\label{fig:stronglyprobabilities}
\end{figure}
The key implication of the construction is that under both distributions, whether the learner plays action $v$ or $u$, she does not observe the loss of $v$ with probability $1/4$, she observes a loss of $0$ for action $v$ with probability $3/8$ and she observes a loss of $1$ with probability $3/8$. This is although in the first distribution the loss of $v$ is distributed Bernoulli($3/8$), and in the second distribution it is distributed Bernoulli($5/8$). Therefore, the learner can never tell if $u$ or $v$ is the action with the smaller loss. 
\end{proof}
The result above is in contrast to the stochastic setting; not only that the problem is learnable but there is an algorithm that attains $\tO(\sqrt{\alpha T})$ regret---the same regret bound that is obtained in the setting where the learner gets to observe the feedback graph fully at the end of each round. In particular, we have:
\begin{theorem}
In the stochastic setting, there exists an online learning algorithm that attains $\tO(\sqrt{\alpha T})$ regret, provided that $G_1,\ldots,G_T$ are all strongly observable.
\end{theorem}
\vspace{-0ex}
\begin{proof}[Proof sketch]
The algorithm is the same as \cref{alg:stochastic}, where the only difference is in the implementation of \cref{alg:efficientsampling}.
Even if the graph is strongly observable, a subgraph of exactly one vertex might not be observable since this vertex might not have a self-loop. Therefore, we stop the while loop in \cref{alg:efficientsampling} when $V_r$ has one action or less. If it is left with exactly one action, say $v$, we add a different arbitrary action from $V$ and we sample uniformly at random from both actions until the loss of $v$ is observed. The probability of observing $v$ is at least $1/2$, and therefore the modification only adds a constant number of rounds in expectation to the total runtime of \cref{alg:efficientsampling}.
\end{proof}

\subsection{Observable graphs}

For the adversarial setting, the problem is unlearnable since it is already unlearnable for strongly observable graphs by \cref{thm:stronglylowerbound}.
On the other hand, in the stochastic setting we have the following result.
%
\begin{theorem}
In the stochastic setting, there exists an online learning algorithm that attains $\tO(K^{1/3} T^{2/3})$ regret, provided that $G_1,\ldots,G_T$ are all observable.
\end{theorem}
This regret bound is tight up to logarithmic factors for weakly observable $G_1,...,G_T$, since the $\wt \Omega(K^{1/3} T^{2/3})$ lower bound proved by  \citet{AlonCDK15}, in the easier setting where the graphs are revealed following each decision, applies in our stochastic setting.
\vspace{-0ex}
\begin{proof}[Proof sketch]
The algorithm is done in two phases. In the exploration phase, the learner estimates the means of the losses of all actions to $\epsilon$ accuracy. 
The learner simply plays actions uniformly at random. By an argument similar to that of the coupon collector problem, it suffices for the exploration phase to complete after $\tO(K / \epsilon^2)$ rounds.
In the exploitation phase, the learner plays the best action found during the exploration phase, and suffers an expected instantaneous regret of at most $\epsilon$ per round. 
In total, the expected regret of the learner is $\tO(K / \epsilon^2 + \epsilon T)$. Setting $\epsilon = \tTheta((K / T)^{1/3})$ gives an expected regret bound of $\tO(K^{1/3} T^{2/3})$.
\end{proof}

\section{Additional proofs}
\label{sec:additionalproofs}

\subsection{Proof of \cref{lemma:meanapproximation}}

\begin{proof}
\textsc{AlphaSample} observed the loss of each action in $V_r$ for $n_r$ times. 
Note that by assumption the losses of the actions are distributed independently from the feedback graphs. Therefore by Hoeffding's inequality and the union bound we have w.p.~at least $1 - T^{-1}$,
\[
\forall ~ v \in V_r ~,
\quad
|\mu(v) - m_r(v)| \le \epsilon_r ~.
\]
Denote by $\tilde v$ an action such that $m_r(\tilde v) = m^\star_r$. Note that by induction $v^\star \in V_{r+1}$ since if $v^\star \in V_r$ then
\[
m_r(v^\star) \le \mu^\star + \epsilon_r \le \mu(\tilde v) + \epsilon_r \le m^\star_r + 2 \epsilon_r ~.
\]
Therefore, for all $v \in V_{r+1}$ we have
\[
\mu(v) \le m_r(v) + \epsilon_r \le m^\star_r + 3 \epsilon_r \le m_r(v^\star) + 3\epsilon_r \le \mu^\star + 4 \epsilon_r ~. \qedhere
\]
\end{proof}

\subsection{Proof of \cref{lemma:alg2expectedtime}}
\label{sub:proofalg2expectedtime}
\begin{proof}
Notice that \cref{alg:efficientsampling} takes at most $K$ rounds with probability 1.
Using \cref{thm:alg2analysis} with $\delta = \alpha / K$, we get that the expected number of rounds for the algorithm to complete is at most
$
4 \alpha \log (K^2/\alpha) + (\alpha/K) \cdot K \le 10 \alpha \log K ,
$
since $K \ge 2$ by assumption and $\alpha \ge 1$.
\end{proof}

\subsection{Proof of \cref{lemma:rgis}}
\label{sub:proofofislemma}
To prove \cref{lemma:rgis} we will need the following lemma.

\begin{lemma}
\label{lemma:ergis}
Let $G = (V,E)$ be an undirected Erd\"{o}s-R\'{e}nyi graph, such that each edge appears independently with probability $p$. For any $0 \le \delta \le \sqrt{1-p}$, the independence number of $G$ is at most $2 \log_{1/(1-p)} (K / \delta) + 1$ with probability at least $1 - \delta$.
\end{lemma}

\begin{proof}
%
Denote $\tau = 1/(1-p)$ and set $
\alpha = \lceil 2 \log_\tau K + \sqrt{2 \log_\tau (1 / \delta)} \rceil$.
First, consider the concave quadratic $x \log_\tau K - x(x-1)/2 - \log_\tau \delta$. It is negative for any $x > \log_\tau K + 1/2 + \sqrt{\log_\tau^2 K - 2 \log_\tau \delta}$ and in particular it is negative for $x = \alpha + 1$. Thus we have
\begin{equation}
\label{eq:quadraticinequality}
(\alpha + 1) \log_\tau K - \alpha(\alpha + 1) / 2 < \log_\tau \delta
\end{equation}

Next, suppose that the independence number of the graph is more than $\alpha$, and therefore there is an independent set of $\alpha + 1$ vertices.
The probability that a subset $S \subseteq V$ is an independent set is $(1-p)^{\binom{|S|}{2}} = \tau^{-\binom{|S|}{2}}$. 
By a union bound,
\[
\Pr[\alpha(G) > \alpha] 
\le K^{\alpha + 1} \tau^{-\binom{\alpha + 1}{2}} 
= \tau^{(\alpha + 1) \log_\tau K - \alpha(\alpha + 1) / 2} < \tau^{\log_\tau \delta} = \delta ~.
\]
where the second inequality is by \cref{eq:quadraticinequality}.

To finish the proof of the lemma, we have that with probability at least $1 - \delta$, the independence number of the graph is at most
\begin{align*}
\alpha = \left\lceil 2 \log_\tau K + \sqrt{2 \log_\tau \tfrac{1}{\delta}} \right\rceil 
&\le 2 \log_\tau K + 2 \log_\tau \tfrac{1}{\delta} + 1 = 2 \log_\tau \tfrac{K}{\delta} + 1
\end{align*}
since $\delta \le \tau^{-1/2}$ by assumption.
\end{proof}

We now turn to prove \cref{lemma:rgis}.

\begin{proof}[Proof of \cref{lemma:rgis}]
Let $G$ be one of the graphs in the sequence, and consider an undirected version of $G$\footnote{An undirected graph in which there is an edge between $u$ and $v$ if $(u,v) \in E$ or $(v,u) \in E$.}. If we remove $v^\star$ from the graph, we get an Erd\"{o}s-R\'{e}nyi subgraph with each edge appearing with probability $1 - 4\epsilon^2$ and \emph{independently} of other edges. The independence number of the original graph is at most that of the subgraph plus one. Thus, it remains to bound the independence number of the subgraph.

Denote $\tau = 1/ (4 \epsilon^2) = 16 T / K$.
We apply \cref{lemma:ergis} with $\delta = \epsilon/(8T)$ and get that with probability at least $1 - \delta$, the independence number of the subgraph is at most $\alpha = 2 \log_\tau (8 K T / \epsilon) + 1$. 
Then by the choice of $\epsilon$ and since $T \ge K^2$,
\begin{align*}
\alpha = 2 \log_\tau \frac{8 K T}{\epsilon} + 1 &= 2 \log_\tau (64 K^{1/2} T^{3/2}) + 1 \\
&\le 2 \log_\tau \left( \frac{16 T}{K} \right)^{7/2} + 1 \\
&= 2 \cdot \frac{7}{2} + 1= 8 ~,
\end{align*}
Appying this argument to each of the graphs $G_1,...,G_T$, the claim holds by the union bound. 
\end{proof}

\subsection{Proof of \cref{thm:lowerbound}}
\label{sec:lowerboundproof}
To prove the theorem, we shall need a few definitions. Let $\Pr, \Q$ be a couple of distributions over the same space and sigma-algebra $\F$. We define the \emph{total variation distance} between $\Pr$ and $\Q$ as 
\[
\TV{\Pr}{\Q} = \sup_{E \in \F} \left|\Pr[E] - \Q[E] \right|
\]
If $\Pr$ and $\Q$ are discrete distributions, we define the KL divergence between $\Pr$ and $\Q$ as 
\[
\KL{\Pr}{\Q} = \sum_x \log \left( \frac{\Pr[x]}{\Q[x]} \right) \Pr[x]
\]
assuming the support of $\Pr$ is contained in that of $\Q$, and where the sum is taken over the support of $\Pr$.

We can now turn to the proof of the theorem.
\begin{proof}[Proof of \cref{thm:lowerbound}]
Let us introduce the random variables $T_v$ whose value is the number of times the learner plays action $v$. We also introduce the notations $\Pr_v$ and $\E_v$ indicating probability and expectation with respect to the marginal distributions under which $v^\star = v$. Then, we have
\begin{align}
\Regret &= \E \left[ \sum_{t=1}^T \ell_t(v_t) - \sum_{t=1}^T \ell_t(v^\star) \right] \nonumber \\
&= \frac{1}{K} \sum_{v \in V} \E_v \left[ \sum_{t=1}^T \ell_t(v_t) - \sum_{t=1}^T \ell_t(v) \right] \nonumber \\
&= \frac{1}{K} \sum_{v \in V} \epsilon \cdot \E_v[T - T_v] \nonumber \\
&= \epsilon \left( T - \frac{1}{K} \sum_{v \in V} \E_v [T_v] \right) ~,\label{eq:regretlowebound}
\end{align}
and in order to proceed we shall upper bound $\E_v[T_v]$.

Introduce a new distribution, in which the losses of the actions are 
independent Bernoulli($1/2$) variables, and the feedback graphs are such that each directed edge appears with probability $1-2\epsilon$ independently of the other edges and the losses of the actions. 
We will refer to this new law using $\Pr_0$ and $\E_0$. Let $\lambda_t$ be the losses and edges observed at time $t$, and similarly $\lambda^{(t)} = (\lambda_1,...,\lambda_t)$ are the losses and edges observed up until time $t$ (inclusive). Then, since the sequence $\lambda^{(T)}$ determines the actions of the learner over the entire game, and by Pinsker's inequality,
\begin{align}
\E_v[T_v] - \E_0[T_v] &\le T \cdot \TV{\Pr_v[\lambda^{(T)}]}{\Pr_0[\lambda^{(T)}]} \nonumber \\
&\le T \sqrt{\frac{1}{2} \KL{\Pr_0[\lambda^{(T)}]}{\Pr_v[\lambda^{(T)}]}} ~.\label{eq:pinsker}
\end{align}
Moreover, by the chain rule of KL-divergence, $\KL{\Pr_0[\lambda^{(T)}]}{\Pr_v[\lambda^{(T)}]}$ equals
\begin{equation}
\label{eq:klchainrule}
\sum_{t=1}^T \sum_{\lambda^{(t-1)}} \Pr_0[\lambda^{(t-1)}] \KL{\Pr_0[\lambda_t | \lambda^{(t-1)}]}{\Pr_v[\lambda_t | \lambda^{(t-1)}]}~.
\end{equation}

Consider a single term in the sum. Recall that $\lambda^{(t-1)}$ determines the action $v_t$ chosen by the learner on round $t$. If $v_t \neq v$ then, by our construction, the losses and edges of the graph observed by the learner are distributed exactly the same under $\Pr_v$ and $\Pr_0$, and the KL divergence is 0. If $v_t = v$ then the losses of all other actions are distributed Bernoulli($1/2$), and independently of the loss of action $v$ and the observed edges. The latter is so under both $\Pr_v$ and $\Pr_0$. Moreover, the observed edges are distributed Bernoulli($1-2 \epsilon$) independently of the loss of action $v$ under both $\Pr_v$ and $\Pr_0$. Namely,  the only element that is distributed differently under $\Pr_v$ and $\Pr_0$ is the loss of action $v$, and the latter is distributed independently from all other observed variables. Recall that the loss of action $v$ is distributed as Bernoulli($1/2$) under $\Pr_0$ and as Bernoulli($1/2 - \epsilon$) under $\Pr_v$. Therefore, $\KL{\Pr_0[\lambda_t | \lambda^{(t-1)}]}{\Pr_v[\lambda_t | \lambda^{(t-1)}]}$ is upper-bounded by
\begin{align*}
\KL{\frac{1}{2}}{\frac{1}{2} - \epsilon} 
= - \frac{1}{2} \log(1 - 4 \epsilon^2) 
\le 4 \epsilon^2 ~,
\end{align*}
where the last inequality holds since $\epsilon < 1/4$ by assumption.
Plugging the above back into \cref{eq:klchainrule},
\begin{align*}
\KL{\Pr_0[\lambda^{(T)}]}{\Pr_v[\lambda^{(T)}]} &\le \sum_{t=1}^T \Pr_0[v_t = v] 4 \epsilon^2 = 4 \epsilon^2 \E_0[T_v]~,
\end{align*}
and the latter into \cref{eq:pinsker}, we get that $\E_v[T_v] \le \E_0[T_v] + T \epsilon \sqrt{2 \E_0[T_v]}$.

Now, $K \ge 2$ by assumption, and therefore 
\begin{align*}
\frac{1}{K} \sum_{v \in V} \E_v[T_v] &\le \frac{1}{K} \sum_{v \in V} \E_0[T_v] + \frac{1}{K} \sum_{v \in V} T \epsilon \sqrt{2 \E_0[T_v]} \\
&\le \frac{1}{K} \sum_{v \in V} \E_0[T_v] + T \epsilon \sqrt{\frac{1}{K} \sum_{v \in V} 2 \E_0[T_v]} \\
&= \frac{T}{K} + T \epsilon \sqrt{\frac{2 T}{K}} \le \frac{T}{2} + T \epsilon \sqrt{\frac{2 T}{K}} ~.
\end{align*}

Let us now return to \cref{eq:regretlowebound}. We can lower bound the regret as
\begin{align*}
\Regret \ge \epsilon \left(T - \frac{T}{2} - T \epsilon \sqrt{\frac{2 T}{K}} \right) 
= \epsilon T \left(\frac{1}{2} - \epsilon \sqrt{\frac{2 T}{K}} \right) ~.
\end{align*}
By our choice of $\epsilon$, we have that $\epsilon \sqrt{(2 T) / K}$ is at most $1/4$, and so 
\[
\Regret \ge \frac{T}{8}\sqrt{\frac{K}{T}} \left(\frac{1}{2} - \frac{1}{4} \right) = \frac{1}{32} \sqrt{K T} ~,
\]
as claimed.
\end{proof}

\subsection{Proof of \cref{thm:main-lower}}
\label{sub:proofofmainlower}
We shall construct an environment whose distribution over graphs is the same as the environment described in \cref{sec:adver}, conditioned on the event that the independence numbers of all graphs are bounded by 9. 

We now claim that the regret against this environment is not too far off the regret against the original environment. This is because the expected regret against the original environment is at most the expected regret against this environment plus $(\epsilon / 8) T$, by \cref{lemma:rgis} and since the regret is at most $T$ (with probability $1$).

By \cref{thm:lowerbound}, the regret against this environment is at least
\[
\Regret - \frac{\epsilon}{8} T \ge \frac{\sqrt{K T}}{32} - \frac{(1/8)\sqrt{K/T}}{8} T = \frac{\sqrt{K T}}{64}
\]
and thus the lower bound holds.

\subsection{Proof of \cref{thm:stronglylowerbound}}
\label{sub:proofofstronglylowerbound}

\begin{proof}
By Yao's minimax principle, in order to prove a lower bound on the learner's regret it is enough to demonstrate a randomized strategy for the environment that forces any deterministic learner to incur $\Omega(T)$ regret.
We will construct our environment's strategy as follows. 

Consider a learning problem over two actions, $u$ and $v$. Before the game starts, the environment samples an index $\chi \in \{1, 2\}$ uniformly at random. If $\chi = 1$ then the environment plays one distribution; if $\chi = 2$ then she plays another distribution. Under the first distribution, the loss of $v$ is distributed Bernoulli($3/8$); under the second distribution, it is distributed as Bernoulli($5/8$). In both cases the action $u$ always has a self-loop and its loss is constantly $1/2$.

The feedback graphs $G_1,\ldots,G_T$ are chosen i.i.d.~and are dependent on the loss of action $v$.
Under the first distribution, if the loss of $v$ is 1, both an edge $u \rightarrow v$ and the self-loop $v \rightarrow v$ appear with probability $1$; if the loss of $v$ is 0, with probability $2/5$ only the edge $u \rightarrow v$ appears, with probability $2/5$ only the self-loop $v \rightarrow v$ appears and with probability $1/5$ both the edge and the self-loop appear. Under the second distribution, if the loss of $v$ is 1, with probability $2/5$ only the edge $u \rightarrow v$ appears, with probability $2/5$ only the self-loop $v \rightarrow v$ appears and with probability $1/5$ both the edge and the self-loop appear; if the loss of $v$ is 0, both an edge $u \rightarrow v$ and the self-loop $v \rightarrow v$ appear with probability $1$.
See \cref{fig:stronglyprobabilities} for a summary of the edge probabilities in this construction. 
Note that in every case, the action $v$ either have a self-loop or and incoming edge from $u$ (or both). Therefore, the graphs $G_1,...G_T$ are all strongly observable.

The key implication of the construction above is as following. Under both distributions, if the learner plays action $v$, she does not observe the loss of $v$ with probability $1/4$, she observes a loss of $0$ with probability $3/8$ and she observes a loss of $1$ with probability $3/8$. If the learner plays action $u$, she does not observe the loss of $v$ with probability $1/4$, she observes a loss of $0$ with probability $3/8$ and she observes a loss of $1$ with probability $3/8$. Moreover, if the learner plays action $v$, she does not observe whether there is an incoming edge $u \rightarrow v$; if she plays action $u$ she does not observe whether there is a self-loop $v \rightarrow v$.

Let $M$ denote the number of times the learner chooses to play action $v$. By the argument above she must play the exact same strategy under both distributions, and as a consequence the random variable $M$ is independent of the choice of distribution $\chi$. Hence, the expected regret of the learner against the environment constructed above is at least
\begin{align*}
&\half \E \left[\frac{1}{8} (T - M) \bigg| \chi = 1 \right]  + \half \E \left[\frac{1}{8} M \bigg| \chi = 2 \right] 
 = \half \E \left[\frac{1}{8} (T - M) \right]  + \half \E \left[\frac{1}{8} M \right] = \frac{1}{16} T
\end{align*}
as claimed.
\end{proof}

\section*{Acknowledgements} 

TK would like to thank Nicol\`{o} Cesa-Bianchi and Ofer Dekel for stimulating discussions in the early stages of this research.

\bibliographystyle{abbrvnat}
\bibliography{bib}

\end{document}

%% file: header.tex
\usepackage{amsmath,amsfonts,amssymb}
\usepackage{amsthm} 

\usepackage[capitalise]{cleveref}
\crefname{figure}{Figure}{Figures}

\usepackage[table]{xcolor}

\newcommand{\ignore}[1]{}

\theoremstyle{plain}
\newtheorem{theorem}{Theorem}
\newtheorem{lemma}[theorem]{Lemma}

\newtheorem*{theorem*}{Theorem}
\newtheorem*{lemma*}{Lemma}
\newtheorem*{corollary*}{Corollary}
\newtheorem*{proposition*}{Proposition}
\newtheorem*{claim*}{Claim}
\newtheorem*{fact*}{Fact}

\theoremstyle{definition}

\newtheorem*{definition*}{Definition}

\newtheorem*{remark*}{Remark}

\newtheorem*{example*}{Example}

\theoremstyle{plain}
\newtheorem*{theoremaux}{\theoremauxref}
\gdef\theoremauxref{1}

\DeclareMathAlphabet{\mathbfsf}{\encodingdefault}{\sfdefault}{bx}{n}

\let\Pr\relax
\DeclareMathOperator{\Pr}{\mathbb{P}}

\newcommand{\lr}[1]{\!\left(#1\right)\!}

\newcommand{\lrbra}[1]{\!\left[#1\right]\!}

\newcommand{\set}[1]{\{#1\}}

\newcommand{\wt}[1]{\smash{\widetilde{#1}}}

\renewcommand{\O}{O}

\newcommand{\tO}{\wt{\O}}
\newcommand{\tTheta}{\wt{\Theta}}
\newcommand{\E}{\mathbb{E}}
\newcommand{\EE}[1]{\E\lrbra{#1}}

\newcommand{\D}{\mathcal{D}}

\newcommand{\Q}{\mathbb{Q}}
\newcommand{\F}{\mathcal{F}}

\newcommand{\st}{\star}

\newcommand{\half}{\frac{1}{2}}

\newcommand{\eq}{~=~}

\let\nablaold\nabla
\renewcommand{\nabla}{\nablaold\mkern-2.5mu}